\documentclass[letterpaper,10pt,conference]{ieeeconf}
\IEEEoverridecommandlockouts
\usepackage{amsmath,amssymb}
\usepackage{bm}

\usepackage{amsthm}
\usepackage{cite}
\usepackage{graphicx}
\usepackage{url}
\usepackage{tikz}
\usepackage{tikz-qtree}
\usetikzlibrary{positioning,automata,backgrounds}
\usepackage{xcolor}
\usepackage{algorithm}
\usepackage[noend]{algpseudocode} % algorithmicx
\usepackage{subcaption} 
\usepackage{mathtools}
\usepackage{hyperref}
\usepackage{booktabs}
\usepackage{multirow}

\definecolor{mumred}{RGB}{222,33,77}
\definecolor{mumgreen}{RGB}{0,140,0}
\definecolor{mumblue}{RGB}{0,100,222}
\definecolor{mumpurple}{RGB}{128,0,128}

\theoremstyle{plain}
\newtheorem{theorem}{Theorem}
\newtheorem{lemma}{Lemma}

\newtheorem{corollary}{Corollary}
\newtheorem{definition}{Definition}

\newtheorem{problem}{Problem}
\newtheorem{remark}{Remark}

\newcommand{\xb}{\bm{x}}
\newcommand{\ub}{\bm{u}}
\newcommand{\fb}{\bm{f}}
\newcommand{\ab}{\bm{a}}
\newcommand{\lambdab}{\bm{\lambda}}
\newcommand{\hl}[1]{#1}

\newcommand{\until}{{\bf{U}}}
\newcommand{\always}{\square}
\newcommand{\eventually}{\lozenge}
\newcommand{\subheight}{10mm} 

\begin{document}

\title{Exact Smooth Reformulations for Trajectory Optimization \\ Under Signal Temporal Logic Specifications}

\author{Shaohang Han, Joris Verhagen, and Jana Tumova%
\thanks{This work was partially supported by the Wallenberg AI, Autonomous Systems and Software Program (WASP) funded by the Knut and Alice Wallenberg Foundation.}%
\thanks{The authors are with the Division of Robotics, Perception and Learning, School of Electrical Engineering and Computer Science, KTH Royal Institute of Technology, Stockholm, Sweden, and are also affiliated with Digital Futures. Email: \texttt{\{shaohang,~jorisv,~tumova\}@kth.se}.}%
}

\maketitle
\thispagestyle{empty}
	
\begin{abstract}
We study motion planning under Signal Temporal Logic (STL), a useful formalism for specifying spatial-temporal requirements. We pose STL synthesis as a trajectory optimization problem leveraging the STL robustness semantics. To obtain a differentiable problem without approximation error, we introduce an exact reformulation of the $\max$ and $\min$ operators. The resulting method is exact, smooth, and sound. We validate it in numerical simulations, demonstrating its practical performance.
\end{abstract}

\section{Introduction}
Temporal logics are formal specification languages used to describe complex desired behaviors of autonomous systems. One such language is Signal Temporal Logic (STL)~\cite{maler2004monitoring}, which supports quantitative temporal properties (e.g., the robot should be in region $A$ every 10–20 seconds and in region $B$ every 15–25 seconds). From a motion planning standpoint, STL is particularly appealing because it admits quantitative semantics via the notion of \textit{robustness}. The robustness measures the \textit{degree} to which a trajectory satisfies a given STL specification. Through this unique property, we can formulate an optimization problem to find a maximally robust trajectory that satisfies the STL specification \cite{belta2019formal}. Therefore, STL trajectory optimization has become an active research topic in robotic applications ranging from mobile robots \cite{verhagen2025collaborative} to manipulators \cite{kurtz2020trajectory,takano2021continuous} and bipedal robots \cite{gu2025robust}.

The robustness semantics rely on $\max$ and $\min$ operators to capture logical disjunctions, conjunctions, and temporal ambiguity. However, these operators are challenging to embed in optimization problems because they are non-smooth and can be non-convex. A common approach is to encode them using Mixed-Integer Programming (MIP) with the help of binary variables. In particular, mixed-integer convex programming (MICP) formulations~\cite{belta2019formal,kurtz2022mixed,7039363,verhagen2024temporally} are guaranteed to be complete, sound, and globally optimal for systems with linear dynamics and convex predicates. Despite their effectiveness in this regime, these formulations do not readily extend to problems with nonlinear dynamics or nonlinear predicates. Their combinatorial complexity due to binary variables becomes prohibitive for high-dimensional systems, long planning horizons, or complex formulas.

To circumvent the non-smoothness issue, researchers have also explored sampling-based trajectory optimization methods, such as Model Predictive Path Integral (MPPI) \cite{varnai2021two,11127582}. This type of method iteratively refines trajectories using estimated gradients computed from random rollouts. However, these estimates can exhibit high variance, which could lead to unstable convergence, making them scale poorly to long horizons.

An alternative is to utilize approximate smooth robustness degree~\cite{pant2017smooth,gilpin2020smooth,welikala2023smooth,mehdipour2019arithmetic}, where the $\max$ and $\min$ operators are replaced by smooth approximations. The resulting smooth functions admit derivatives, making them amenable to Nonlinear Programming (NLP) formulations that can be solved using off-the-shelf derivative-based numerical solvers. Specifically, \cite{pant2017smooth} proposes log-sum-exponential approximations, which are smooth everywhere but not sound. In contrast, \cite{mehdipour2019arithmetic} introduces arithmetic–geometric mean robustness, which is sound but not smooth everywhere. The authors in \cite{gilpin2020smooth} propose different approximations that are both sound and smooth. In \cite{gilpin2020smooth}, the robustness degree with the smooth specification is the under-approximation for the true robustness of the original specification.

Such under-approximation, however, may introduce a sub-optimality gap and can shrink the feasible set when robustness is posed as a constraint. Moreover, \cite{gilpin2020smooth} relies on hyperparameters that trade off approximation accuracy against numerical conditioning, which requires additional tuning effort. Motivated by this, we instead propose \emph{exact} smooth reformulations of the $\max$ and $\min$ operators for STL trajectory optimization that preserve the true robustness while ensuring differentiability. This removes the error in prior smooth approximations and does not require any tuning parameters. While the reformulation introduces additional variables and constraints, we show in our experiments that the computational overhead can be mitigated by effectively warm-starting the solvers.

This paper is inspired by \cite{wehbeh2025smooth}, which introduces an exact smooth reformulation for logic-constrained optimization. We adapt and extend these ideas to STL trajectory optimization, covering both Boolean and temporal operators that allow nesting. To summarize, our contributions are:
\begin{itemize}
    \item An analysis of the smooth approximation in \cite{gilpin2020smooth}, showing that its error cannot be eliminated in general.
    \item An exact smooth reformulation for STL trajectory optimization that enables derivative-based NLP solvers.
    \item Empirical validation in simulations, demonstrating efficiency and practicality, along with an open-source codebase for the community.
\end{itemize}

\section{Preliminaries And Problem Formulation}
\label{sec:preliminary_problem_formulation}
In this paper, we consider the nonlinear discrete-time system
\begin{equation}
\label{eq:system}
    \xb_{t+1} = \fb(\xb_t, \ub_t),
\end{equation}
where ${\xb_t \in \mathcal{X} \subseteq \mathbb{R}^n}$ is the system state and ${\ub_t \in \mathcal{U} \subseteq \mathbb{R}^m}$ is a control input. 
We assume that the discrete-time dynamics ${\fb:\mathbb{R}^{n} \times \mathbb{R}^{m} \mapsto\mathbb{R}^{n}}$ is smooth in its arguments. In the following, we consider time steps $0$ to $T$ with state trajectory ${\xb\coloneq[\xb_0, \xb_1, ...,\xb_T]^\top \in \mathbb{R}^{n \times (T+1)}}$ and input trajectory ${\ub\coloneq[\ub_0, \ub_1, ...,\ub_{T}]^\top \in \mathbb{R}^{m \times {(T+1)}}}$.

\subsection{Signal Temporal Logic}
We consider time-bounded  STL formulas over nonlinear predicates, which can be defined as:
\begin{definition}[Time-bounded STL] Let $I=[t_1, t_2] \subset \mathbb{Z}_{\ge 0}
$ be a closed bounded time interval, where $t_1\leq t_2$. STL formulas can be recursively written as
\begin{equation*}
\varphi ::= \mu \;\big|\; \neg  \varphi \;\big|\; \land_i \varphi_i \;\big|\; \lor_i \varphi_i \;\big|\; 
\always_{I} \varphi \;\big|\; \eventually_{I} \varphi \;\big|\; 
\varphi_1 \until_{I} \varphi_2,
\end{equation*}
where $\mu := h^{\mu}(x) \geq 0$ is a predicate with smooth function ${h^{\mu}:\mathcal{X} \mapsto \mathbb{R}}$. $\neg$ stands for negation. $\land$ and $\lor$ denote Boolean operations ``and'' and ``or''. $\Box_I$ and $\diamondsuit_I$ stand for the ``always'' and ``eventually'' operators meaning that $\varphi$ should hold $\forall t\in I$ or $\exists t\in I$ respectively. $\until_{I}$ represents the ``until'' operator which specifies $\varphi_1$ should hold until, within $I$, $\varphi_2$ should hold.
\end{definition}

STL formulas are associated with \textit{quantitative} semantics such as \textit{spatial robustness} \cite{maler2004monitoring}, which indicates the degree of satisfaction or violation of the specification. The robustness function evaluated at time $t$ for a trajectory $\xb$ and formula $\varphi$ is denoted as $\rho^\varphi(\xb, t)\in \mathbb{R}$. If a trajectory $\xb$ satisfies the formula $\varphi$, we denote that $\xb \models \varphi$. The semantics can then be recursively defined as follows:

\begin{definition}[STL Robustness Semantics]
\begingroup
% \small                     % optional: smaller font
\setlength{\jot}{3pt}      % line spacing between equations (default: ~3pt)
\setlength{\abovedisplayskip}{2pt}    % space above math env
\setlength{\belowdisplayskip}{2pt}    % space below math env
\setlength{\abovedisplayshortskip}{0pt}
\setlength{\belowdisplayshortskip}{0pt}

\begin{align*}
& \xb \models \varphi \Leftrightarrow \rho^{\varphi}(\xb,0) \ge 0 \\
& \rho^{\mu}(\xb,t) = h^{\mu}(\xb_t) \\
& \rho^{\neg \varphi}(\xb,t) = -\rho^{\varphi}(\xb,t) \\
& \rho^{\varphi_1 \wedge \varphi_2}(\xb,t)
    = \min\!\big([\rho^{\varphi_1}(\xb,t),\,\rho^{\varphi_2}(\xb,t)]^\top\big) \\
& \rho^{\varphi_1 \vee \varphi_2}(\xb,t)
    = \max\!\big([\rho^{\varphi_1}(\xb,t),\,\rho^{\varphi_2}(\xb,t)]^\top\big) \\
& \rho^{\Diamond_I \varphi}(\xb,t)
    = \max\!\big([\rho^{\varphi}(\xb,t')]_{t' \in t+I}^\top\big) \\
& \rho^{\Box_I \varphi}(\xb,t)
    = \min\!\big([\rho^{\varphi}(\xb,t')]_{t' \in t+I}^\top\big) \\
& \rho^{\varphi_1 \until_I \varphi_2}(\xb,t)
  = \max\!\Big(
\\[-2pt]&\qquad
    \big[
      \min\!\big(
        [\,\rho^{\varphi_2}(\xb,t'),\; [\rho^{\varphi_1}(\xb,t'')]_{t'' \in [t,t']}^\top\,]^\top
      \big)
    \big]_{t' \in t+I}^\top
  \Big).
\end{align*}
\endgroup
\end{definition}
In the following, we use $\rho^{\varphi}(\xb)$ as shorthand for $\rho^{\varphi}(\xb,0)$. 

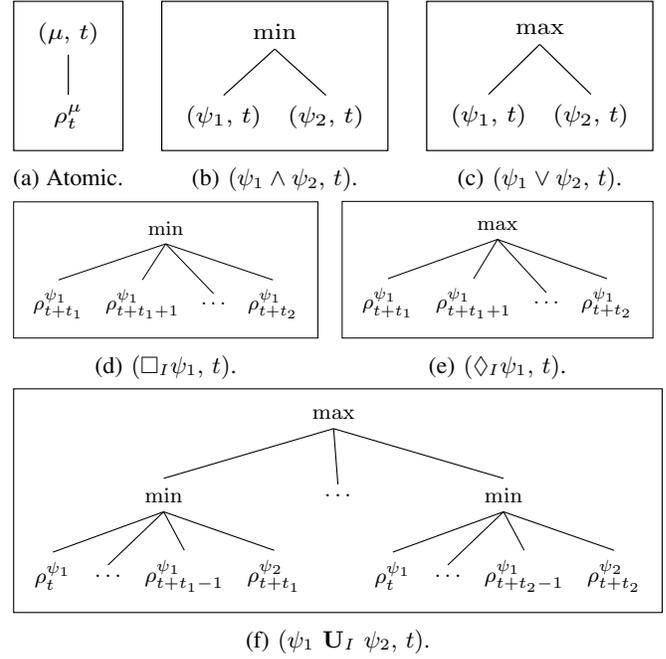
\begin{figure}[t]
\centering
\tikzset{
  every tree node/.style={font=\footnotesize}, % smaller nodes
  % optionally: level distance=10mm, sibling distance=8mm,
}
% --- Row 1: three across ---
\begin{subfigure}{.17\linewidth}
\centering
\resizebox{\linewidth}{\subheight}{%
\begin{tikzpicture}[framed, baseline=(current bounding box.center)]
  \Tree [.{${(\mu,\,t)}$} \hl{$\rho^{\mu}_t$} ]
\end{tikzpicture}}
\caption{Atomic.}
\label{fig:rt-atomic}
\end{subfigure}\hfill
\begin{subfigure}{.35\linewidth}
\centering
\resizebox{\linewidth}{\subheight}{%
\begin{tikzpicture}[framed, baseline=(current bounding box.center)]
  \Tree [.{${\min}$}
           {${(\psi_1,\,t)}$}
           {${(\psi_2,\,t)}$}
        ]
\end{tikzpicture}}
\caption{$(\psi_1 \land \psi_2,\,t)$.}
\label{fig:rt-and}
\end{subfigure}\hfill
\begin{subfigure}{.35\linewidth}
\centering
\resizebox{\linewidth}{\subheight}{%
\begin{tikzpicture}[framed, baseline=(current bounding box.center)]
  \Tree [.{${\max}$}
           {${(\psi_1,\,t)}$}
           {${(\psi_2,\,t)}$}
        ]
\end{tikzpicture}}
\caption{$(\psi_1 \lor \psi_2,\,t)$.}
\label{fig:rt-or}
\end{subfigure}
\vspace{1mm}

% --- Row 2: two across ---
\begin{subfigure}{.47\linewidth}
\centering
\resizebox{\linewidth}{!}{%
\begin{tikzpicture}[framed, baseline=(current bounding box.center)]
  \Tree [.{${\min}$}
           \hl{$\rho^{\psi_1}_{t+t_1}$}
           \hl{$\rho^{\psi_1}_{t+t_1+1}$}
           [.{${\dots}$} ]
           \hl{$\rho^{\psi_1}_{t+t_2}$}
        ]
\end{tikzpicture}}
\caption{$(\always_{I}\psi_1,\,t)$.}
\label{fig:rt-always}
\end{subfigure}\hfill
\begin{subfigure}{.48\linewidth}
\centering
\resizebox{\linewidth}{!}{%
\begin{tikzpicture}[framed, baseline=(current bounding box.center)]
  \Tree [.{${\max}$}
           \hl{$\rho^{\psi_1}_{t+t_1}$}
           \hl{$\rho^{\psi_1}_{t+t_1+1}$}
           [.{${\dots}$} ]
           \hl{$\rho^{\psi_1}_{t+t_2}$}
        ]
\end{tikzpicture}}
\caption{$(\eventually_{I}\psi_1,\,t)$.}
\label{fig:rt-eventually}
\end{subfigure}
\vspace{1mm}

% --- Row 3: single wide ---
\begin{subfigure}{1\linewidth}
\centering
\resizebox{\linewidth}{!}{%
\begin{tikzpicture}[framed, baseline=(current bounding box.center)]
  \Tree [.{${\max}$}
    [.{${\min}$}
      \hl{$\rho^{\psi_1}_t$}
      [.{${\dots}$} ]
      \hl{$\rho^{\psi_1}_{t+t_1-1}$}
      \hl{$\rho^{\psi_2}_{t+t_1}$}
    ]
    [.{${\dots}$} ]
    [.{${\min}$}
      \hl{$\rho^{\psi_1}_t$}
      [.{${\dots}$} ]
      \hl{$\rho^{\psi_1}_{t+t_2-1}$}
      \hl{$\rho^{\psi_2}_{t+t_2}$}
    ]
  ]
\end{tikzpicture}}
\caption{$(\psi_1\ \until_{I}\ \psi_2,\,t)$.}
\label{fig:rt-until}
\end{subfigure}
\caption{Building blocks used to construct a robustness tree.}
\label{fig:robustness-tree-subfigs}
\end{figure}

\subsection{STL Robustness Tree}
For discrete-time systems, STL robustness can be parsed using a tree structure as introduced in \cite{kurtz2022mixed,takayama2025stlccp}. In this robustness tree, the leaf nodes store the robustness degree of an atomic predicate at each time step, and an internal node has only $\max$ or $\min$ type. This differs from the STL syntax tree used in \cite{vasile2017sampling,yu2024continuous}, which contains \textit{set} nodes representing the feasible sets specified by subformulas over time intervals. The STL robustness tree can be formally defined as follows and illustrated in Fig.~\ref{fig:robustness-tree-subfigs}.

\begin{definition}[STL Robustness Tree]
The \emph{robustness tree} $\mathcal{T}^{\varphi}$ is a finite rooted, ordered tree whose nodes are labeled by pairs $(\psi,t)$, where $\psi$ is a subformula of $\varphi$ and $t$ is a time index. Every internal node is typed as $\max$ or $\min$ according to the robustness semantics. The tree is constructed recursively:
\begin{itemize}
\item
If $\psi$ is an atomic predicate $\mu$, then $(\mu,t)$ is a leaf. Its value is the predicate robustness $\rho^{\mu}(\xb,t)\in\mathbb{R}$.
\item 
If $\psi=\psi_1\land\psi_2$, then $(\psi,t)$ is a $\min$-node with two children $(\psi_1,t)$ and $(\psi_2,t)$.
\item 
If $\psi=\psi_1\lor\psi_2$, then $(\psi,t)$ is a $\max$-node with two children $(\psi_1,t)$ and $(\psi_2,t)$.
\item 
If $\psi=\always_{I}\,\psi_1$, then $(\psi,t)$ is a $\min$-node with children $\{(\psi_1,t+t')\}_{t'\in I}$.
\item 
If $\psi=\eventually_{I}\,\psi_1$, then $(\psi,t)$ is a $\max$-node with children $\{(\psi_1,t+t')\}_{t'\in {I}}$.
\item 
If $\psi=\psi_1\,\until_{I}\,\psi_2$, then $(\psi,t)$ is a $\max$-node with one child per $t'\in {I}$; its $t'$-th child is a $\min$-node with children ${\{(\psi_1,t+j)\}_{j=0}^{t'-1}}$ and ${(\psi_2,t+t')}$.
\end{itemize}
\end{definition}

By traversing the robustness tree in a depth-first order, we can either embed the STL specification into an optimization problem or compute the robustness degree for a given trajectory, as shown in \cite{kurtz2022mixed}.

\subsection{Problem Formulation}
Given a nonlinear system~\eqref{eq:system}, an STL specification $\varphi$, and an initial state $\xb_0$, we seek a state trajectory $\xb$ and an input trajectory $\ub$ that satisfy $\varphi$, while optimizing for a user-defined objective function.
Formally, the motion planning problem can be formulated as the following trajectory optimization problem under the STL specification.
\begin{problem}
Given a nonlinear system~\eqref{eq:system}, an initial state $\xb_0$, and an STL specification $\varphi$, solve
\begin{subequations}\label{problem_1}
\begin{align}
\min_{\xb, \ub}\quad & - \alpha \rho^{\varphi}(\xb) + \sum_{t=0}^{T} \xb_t^{\top} \bm{Q} \xb_t + \ub_t^{\top} \bm{R} \ub_t \label{problem_1:obj}\\
\text{s.t.} \quad & \forall t \in[0, T-1]: \xb_{t+1} = \fb(\xb_t, \ub_t), \label{problem_1:dyn} \\
& \xb_0 \;\text{fixed}, \label{problem_1:init}\\
& \forall t \in[0, T]: \xb_t \in \mathcal{X}, \; \ub_t \in \mathcal{U}, \label{problem_1:sets}\\
& \rho^{\varphi}(\xb) \ge 0 \label{problem_1:rob},
\end{align}
\end{subequations}
where $\alpha$ is a positive weight, $\bm{Q} \succeq 0$ and $\bm{R} \succeq 0$ are symmetric weighting matrices.
\end{problem}

\begin{remark}
Compared to a standard trajectory optimization problem, \eqref{problem_1} augments the objective~\eqref{problem_1:obj} with the term \(-\alpha \rho^{\varphi}(\xb)\) to maximize the degree of STL satisfaction, while also including standard performance terms such as control-effort minimization. When $\bm{Q}$ and $\bm{R}$ are both $\bm{0}$, \eqref{problem_1} solves for a maximally robust trajectory. The constraint~\eqref{problem_1:rob} is added to enforce satisfaction of the formula.
\end{remark}

The function $\rho^{\varphi}(\xb)$ is generally non-smooth due to the $\max/\min$ operators. Smooth under-approximations~\cite{gilpin2020smooth} replace $\rho^{\varphi}(\xb)$ with a smooth surrogate $\tilde{\rho}^{\varphi}(\xb)=\rho^{\varphi}(\xb)-\epsilon(\xb)$, where $\epsilon$ denotes the non-negative error. However, this error cannot be computed without using discrete $\max/\min$ and solving~\eqref{problem_1}. As a result, \cite{gilpin2020smooth} directly uses $\tilde{\rho}^{\varphi}$ in ~\eqref{problem_1}, which prevents accurately quantifying the optimal spatial robustness in~\eqref{problem_1:obj}, and can shrink the feasible set by making~\eqref{problem_1:rob} more conservative.

In the following, we first provide a quantitative analysis of the smooth approximations in~\cite{gilpin2020smooth} and show that the resulting errors are unavoidable in general. We then present our main results on exact reformulations.

\section{Analysis of the Smooth Approximations}
\label{sec:analysis}
Given a vector $\ab=[a_1, \ldots, a_m]^\top$, \cite{gilpin2020smooth} proposes the following under-approximations for $\max$ and $\min$ operators:
\begin{equation}
\label{eq:smooth_op}
\begin{aligned}
&\widetilde{\max}(\ab) = 
\frac{\sum_{i=1}^{m} a_i e^{k a_i}}{\sum_{i=1}^{m} e^{k a_i}}, \\
&\widetilde{\min}(\ab) = -\frac{1}{k} \log \left( \sum_{i=1}^{m} e^{-k a_i} \right), 
\end{aligned}
\end{equation}
where $k>0$. These approximations are widely used due to their soundness and practical simplicity \cite{vahs2023risk,kurtz2020trajectory}. The authors in \cite{gilpin2020smooth,welikala2023smooth} also provide the upper bounds of the resulting approximation errors. Instead, in the following lemma, we establish lower bounds on the approximation errors.

\begin{lemma}
\label{lem:lower_bound}
Given $k>0$ and a vector $\ab=[a_1, \ldots, a_m]^\top$. Without loss of generality, we may assume that $m \geq 2$ and $\ab$ is ordered so that ${a_1\ge\cdots\ge a_m}$, since $\max$, $\min$ and~\eqref{eq:smooth_op} are invariant under permutations of the entries. Let $r$ be the number of entries attaining the maximum (i.e., ${a_1=\cdots=a_r>a_{r+1}}$) and let $s$ be the number of entries attaining the minimum (i.e., $a_{m-s} > a_{m-s+1}=\cdots=a_m$). Let $w_i:=\frac{e^{k a_i}}{\sum_{j=1}^m e^{k a_j}}$. Then
\begin{equation*}
\begin{aligned}
&\Delta_{\max}=\max(\ab)-\widetilde{\max}(\ab) \ge \big(a_1-a_{r+1}\big) \sum_{i=r+1}^m w_i, \\
&\Delta_{\min}=\min(\ab)-\widetilde{\min}(\ab)\geq\frac{1}{k}\log\!\Big(s+\sum_{i=1}^{m-s} e^{-k(a_i-a_m)}\Big).
\end{aligned}
\end{equation*}
\end{lemma}

\begin{proof}
For the maximum, the error is
\begin{equation*}
\begin{aligned}
\Delta_{\max} &=  a_1 - \sum_{i=1}^m w_ia_i =\sum_{i=1}^m w_i(a_1-a_i) \\
&=\sum_{i=r+1}^m w_i(a_1-a_i)\ge{}(a_1-a_{r+1})\sum_{i=r+1}^m w_i.
\end{aligned}
\end{equation*}
For the minimum, the error is
\begin{equation*}
\begin{aligned}
\Delta_{\min}
&= a_m+\frac{1}{k}\log\!\sum_{i=1}^{m} e^{-k a_i}=\frac{1}{k} \log \Bigl(\!{e^{k a_m}} \sum_{i=1}^{m} e^{-k a_i}\Bigr)\\
&= \frac{1}{k}\log\!\sum_{i=1}^{m} e^{-k(a_i-a_m)}\\
&\ge \frac{1}{k}\log\!\Bigl(s+\sum_{i=1}^{m-s} e^{-k(a_i-a_m)}\Bigr),
\end{aligned}
\end{equation*}
which completes the claim.
\end{proof}

By inspecting Lemma~\ref{lem:lower_bound}, we can tell that ${\Delta_{\max}>0}$ whenever $\ab$ contains an $a_{r+1}$ that is smaller than $a_1$. ${\Delta_{\min}>0}$ always holds under the assumption $m\geq 2$. Since the error $\epsilon$ is obtained by recursively combining $\Delta_{\max}$ and $\Delta_{\min}$, as demonstrated in \cite{gilpin2020smooth}, it cannot be eliminated in general.
One way to mitigate the approximation error is to tune the parameter $k$. However, a large $k$ makes the exponentials grow rapidly, leading to numerical instability, whereas a small $k$ could increase the error. 

\section{Exact Smooth Reformulation}
\label{sec:exact-smooth}

Motivated by the shortcomings of~\eqref{eq:smooth_op}, we propose exact, smooth reformulations of the $\max$ and $\min$ operators. We then recursively apply these rules to the STL robustness trees to reformulate the optimization problem in~\eqref{problem_1}.

\subsection{Node-wise Reformulation Rules}
\label{subsec:nodewise}

For any finite index set $\mathcal J$ and real-valued functions
$\{\rho_j\}_{j\in\mathcal J}$, we write
$\min_{j\in\mathcal J}\rho_j$ and $\max_{j\in\mathcal J}\rho_j$ for
pointwise minimum and maximum. We also introduce a continuous variable $\delta \in\mathbb R$ to present the following lemmas.

\begin{lemma}
\label{lem:min-split}
Let $\{\rho_j\}_{j\in\mathcal J}$ be finite. Then
\begin{equation*}
\min_{j\in\mathcal J}\rho_j(\xb)\ \ge\ \delta
\Longleftrightarrow
\rho_j(\xb)\ \ge\ \delta\quad\forall j\in\mathcal J .    
\end{equation*}
\end{lemma}

\begin{lemma}
\label{lem:max-simplex}
Let $\{\rho_j\}_{j\in\mathcal J}$ be finite. Then
\begin{equation*}
\max_{j\in\mathcal J}\rho_j(\xb)\ \ge\ \delta
\Longleftrightarrow
\exists\,\lambdab\in\mathbb{S}^{|\mathcal J|}:\ 
\sum_{j\in\mathcal J}\lambda_j\,\rho_j(\xb)\ \ge\ \delta ,
\end{equation*}
where $\mathbb{S}^{|\mathcal J|}\!=\!\{\lambdab\in\mathbb R^{|\mathcal J|}\mid
\lambdab\ge \bm{0},\ \bm{1}^\top\lambdab=1\}$.
\end{lemma}

\begin{proof}
Let $j^\star\in\arg\max_{j}\rho_j(\xb)$. 
($\Rightarrow$)
Setting $\lambda_{j^\star}=1$ and $\lambda_j=0$ for $j\neq j^\star$, then ${\sum_j\lambda_j \rho_j = \rho_{j^\star} \geq \delta}$.
($\Leftarrow$) 
Since $ \sum_j\lambda_j=1$, we have ${\rho_{j^\star} = \sum_j\lambda_j\rho_{j^\star} \geq \sum_j\lambda_j\rho_j}$. Therefore, $\max_{j}\rho_j=\rho_{j^\star} \geq \delta$.
\end{proof}

\begin{algorithm}[t]
\caption{Depth-First Traversal}
\label{alg:tree_traversal}
\begin{algorithmic}[1]
\Require Robustness tree $\mathcal T^\varphi$
\Ensure Constraints $\mathcal C$; auxiliary variables $\mathcal V_\rho$ and $\mathcal V_\lambda$; root robustness variable $\rho_r$
\State $\mathcal C \gets \emptyset, \ \mathcal V_\rho \gets \emptyset, \ \mathcal V_\lambda \gets \emptyset$
\Function{Traverse}{$v$} \Comment{$v=(\psi,t)$}
  \State create scalar variable $\rho_v$; add $\rho_v$ to $\mathcal V_\rho$
  \If{$v$ is leaf $(\mu,t)$}
      \State add $h^{\mu}(\xb_t) \geq \rho_v$ to $\mathcal C$
  \ElsIf{$v$ is $\min$-node} \Comment{Lemma~\ref{lem:min-split}}
      \ForAll{$u \in \mathrm{ch}(v)$}
          \State $\rho_u \gets \Call{Traverse}{u}$
          \State add $\rho_u \ge \rho_v$ to $\mathcal C$
      \EndFor
  \ElsIf{$v$ is $\max$-node} \Comment{Lemma~\ref{lem:max-simplex}}
      \State let $\{u_j\}_{j=1}^{m} \gets \mathrm{ch}(v)$
      \For{$j=1$ to $m$} \State $\rho_{u_j} \gets \Call{Traverse}{u_j}$ \EndFor
      \State create $\lambdab^v \in \mathbb R^{m}$; add $\lambdab^v$ to $\mathcal V_\lambda$
      \State add $\lambdab^v \ge 0$ and $\mathbf 1^\top \lambdab^v = 1$ to $\mathcal C$
      \State add $\sum_{j=1}^{m} \lambda^v_j\,\rho_{u_j} \ge \rho_v$ to $\mathcal C$
  \EndIf
  \State \Return $\rho_v$
\EndFunction
\State $\rho_r  \gets \Call{Traverse}{root}$
\State add $\rho_r \ge 0$ to $\mathcal C$
\State \Return $(\mathcal C,\ \mathcal V_\rho,\ \mathcal V_\lambda,\ \rho_r)$
\end{algorithmic}
\end{algorithm}

\subsection{Recursive Reformulation on the Robustness Tree}
\label{subsec:tree-recursive}
We reformulate the STL specification in~\eqref{problem_1} as a collection of smooth constraints with auxiliary variables. The procedure is summarized in Algorithm~\ref{alg:tree_traversal}. In this algorithm, we traverse the robustness tree $\mathcal T^{\varphi}$ in depth-first order. At leaf nodes, we encode the robustness using predicate functions (lines 4-5); at internal nodes, we apply Lemmas~\ref{lem:min-split} and \ref{lem:max-simplex} node-wise (lines 6-16). 
We denote by $\mathcal V_\rho$ and $\mathcal V_\lambda$ the collections of auxiliary variables introduced by the traversal, and by $\mathcal C$ the collection of all resulting constraints. The next lemma shows that this reformulation preserves the feasible set.

\begin{lemma}
\label{lem:feas-equivalence}
Let ${\mathcal F_{\mathrm{org}} := \{\xb \mid \rho^{\varphi}(\xb)\ge 0\}}$ be the original feasible set, and let $\mathcal C$ be the collection of constraints returned by Algorithm~\ref{alg:tree_traversal}.
We can define
\begin{equation*}
\begin{aligned}
\mathcal F_{\mathrm{new}}
= \bigl\{\, \xb \ \big|\ &
\exists\,(\mathcal V_\rho,\mathcal V_\lambda, \rho_r)\ \text{such that} \\
& (\xb,\mathcal V_\rho,\mathcal V_\lambda,\rho_r)\ \text{satisfy }\mathcal C \bigr\}.
\end{aligned}
\end{equation*}
Then $\mathcal F_{\mathrm{new}}=\mathcal F_{\mathrm{org}}$.
\end{lemma}

\begin{proof}
By construction, the constraints $\mathcal C$ implement the semantics of $\rho^{\varphi}$. If $\xb\in\mathcal F_{\mathrm{org}}$, by Lemma~\ref{lem:min-split} and~\ref{lem:max-simplex}, there exist $(\mathcal V_\rho, \mathcal V_\lambda,\rho_r)$ satisfying $\mathcal C$, hence $\xb\in\mathcal F_{\mathrm{new}}$.
Conversely, if $\xb\in\mathcal F_{\mathrm{new}}$, there exist $(\mathcal V_\rho, \mathcal V_\lambda,\rho_r)$ such that
$(\xb,\mathcal V_\rho,\mathcal V_\lambda,\rho_r)$ satisfy $\mathcal C$ with $\rho_r\ge 0$,
which implies $\xb\in\mathcal F_{\mathrm{org}}$.
\end{proof}

\begin{theorem}[Soundness]
For any discrete-time trajectory $\xb$ and time-bounded STL formula $\varphi$, if there exist $(\mathcal V_\rho, \mathcal V_\lambda,\rho_r)$ such that the constraint set $\mathcal{C}$ is satisfied, then $\xb$ satisfies the specification, i.e. $ \xb \models \varphi $.
\end{theorem}
\begin{proof}
By Lemma~\ref{lem:feas-equivalence}, feasibility of ${(\xb,\mathcal V_\rho,\mathcal V_\lambda,\rho_r)}$ with ${\rho_r\ge 0}$
implies ${\rho^{\varphi}(\xb)\geq 0}$, hence $\xb\models\varphi$.
\end{proof}

For a node with $m$ children, Algorithm~\ref{alg:tree_traversal} introduces $O(m)$ auxiliary variables for a $\min$-node and $O(2m)$ for a $\max$-node, so the total number of added variables and constraints grows linearly with the size of the robustness tree, thus polynomially in the STL formula size. Compared to \cite{gilpin2020smooth}, our method adds $m$ more $\lambdab$-variables at each $\max$-node, but the reformulated problem remains smooth and preserves the original feasible set. In the experiments, we show that our method has comparable running time to \cite{gilpin2020smooth} in most cases, while yielding better optimal values.

As every $\min/\max$ operator introduces auxiliary variables, the form in which the formula is presented to Algorithm~\ref{alg:tree_traversal} impacts the efficiency of the embedding.
As such, we apply the formula flattening technique from \cite{kurtz2022mixed} to simplify the robustness trees. It recursively removes consecutive internal nodes of the same type while maintaining equivalence of the STL formula. 

\subsection{Final Smooth NLP}
The final reformulated smooth NLP is given by:
\begin{equation}
\label{eq:final-nlp}
\begin{aligned}
\min_{\xb, \ub,  \mathcal V_\rho, \mathcal V_\lambda, \rho_r} \; & -\alpha \rho_r + \sum_{t=0}^{T} \xb_t^{\top} \bm{Q} \xb_t + \ub_t^{\top} \bm{R} \ub_t  \\
\text{s.t.} \quad & \forall t \in[0, T-1]: \xb_{t+1} = \fb(\xb_t, \ub_t), \\
& \xb_0 \;\text{fixed}, \\
& \forall t \in[0, T]: \xb_t \in \mathcal{X}, \; \ub_t \in \mathcal{U}, \\
& \text{$(\xb,\mathcal V_\rho,\mathcal V_\lambda,\rho_r)$ satisfy all constraints in $\mathcal C$.}
\end{aligned}
\end{equation}
Here $\xb,\ub$ are the original decision variables, and $(\mathcal V_\rho,\mathcal V_\lambda,\rho_r)$ are auxiliary decision variables introduced by the exact smooth reformulation. 

\begin{corollary}
\label{col:opt_val}
The reformulated problem~\eqref{eq:final-nlp} has the same global optimal value as the original problem~\eqref{problem_1}.
\end{corollary}

\begin{proof}
By Lemma~\ref{lem:feas-equivalence}, the feasible sets are identical. Since the objective function is unchanged, the global optimal values coincide.
\end{proof}

\begin{remark}
From an optimization standpoint, we do not show that the reformulation preserves the correspondence of the global minimizer. Thus, although Corollary~\ref{col:opt_val} ensures that the optimal values coincide, we do not formally guarantee that a numerical solver will find it.
\end{remark}

A possible route to certify the correspondence of minimizer is to convert the formulas to conjunctive normal form (CNF) and encode them as max–min constraints \cite{wehbeh2025smooth,kirjner1998conversion}. However, this transformation increases formula size and, in turn, the runtime of nonlinear numerical solvers. Instead, our experiments indicate that our direct reformulation works well with off-the-shelf smooth NLP solvers and consistently yields trajectories with lower optimal values than the approximation-based baseline \cite{gilpin2020smooth}. Moreover, with good warm starts, we often observe convergence to the global optimum, as shown in Section~\ref{sec:linear_exp}.

To warm-start $(\mathcal V_\rho,\mathcal V_\lambda,\rho_r)$ given a reference trajectory $\xb_{\mathrm{ref}}$, we traverse $\mathcal T^\varphi$, similar to Algorithm~\ref{alg:tree_traversal}. For each node $v$, we evaluate robustness on $\xb_{\mathrm{ref}}$ using the discrete $\max$ and $\min$, and record $\rho_v$ (with $\rho_r$ at the root). Since $\bm{\lambda}$ is only introduced for $\max$ operators, for each $\max$-node $v$ with children $\mathrm{ch}(v)=\{u_j\}_{j=1}^m$, we set $j^\star=\arg\max_j \rho_{u_j}$, then define $\bm{\lambda}^v\in\mathbb{R}^m$ by $\lambda^v_{j^\star}=1$ and $\lambda^v_j=0$ for $j\neq j^\star$. Lastly, we store $\rho_v$ in $\mathcal V_\rho$ and $\bm{\lambda}^v$ in $\mathcal V_\lambda$. 

\section{Numerical Experiments}
\label{sec:exp}
We validate the proposed method in both linear benchmarks and a nonlinear case study, all involving a mobile robot navigating in planar environments. Our Python implementation\footnote{Code available at: \url{https://github.com/KTH-RPL-Planiacs/stl_smooth_reformulation}} builds on the code in \cite{kurtz2022mixed} and uses Drake’s \cite{tedrake2019drake} mathematical programming interface. We use SNOPT \cite{gill2005snopt}, a Sequential Quadratic Programming (SQP) solver, to solve the NLPs. All experiments were run on a laptop with an Intel Core i7-13700H CPU and 32 GB RAM. 

\subsection{Linear Benchmarks}
\label{sec:linear_exp}
We benchmark on the linear-dynamics, linear-predicate scenarios from \cite{kurtz2022mixed}. The robot follows discrete-time double-integrator dynamics
$$
\xb_t=[p_{x_t},p_{y_t},\dot p_{x_t},\dot p_{y_t}]^\top \in \mathbb{R}^4,\quad
\ub_t=[\ddot p_{x_t},\ddot p_{y_t}]^\top \in \mathbb{R}^2,
$$
where $p_{x_t}$ and $p_{y_t}$ are the horizontal and vertical positions. The dynamics are
$$
\xb_{t+1} = \bm{A}\xb_t + \bm{B}\ub_t,
\qquad
\bm{A} = \begin{bmatrix} \bm{I}_2 & \bm{I}_2 \\[2pt] \bm{0} & \bm{I}_2 \end{bmatrix}, \quad
\bm{B} = \begin{bmatrix} \bm{0} \\[2pt] \bm{I}_2 \end{bmatrix},
$$
where $\bm{I}_2$ and $\bm{0}$ are the $2\times 2$ identity and zero matrices, respectively. We use the initial states from \cite{kurtz2022mixed}, set $\alpha=1$, $\bm Q=\operatorname{diag}(0,0,1,1)$, $\bm R=\bm I_2$, and vary the time horizon $T$.

We compare our method to the NLP using smooth approximations \cite{gilpin2020smooth} and the MICP formulation \cite{belta2019formal}. To warm-start the NLPs across different time horizons, we linearly interpolate the MICP solutions computed at $T=25$. For \cite{gilpin2020smooth}, we use grid search to find the parameter $k$ that minimizes the optimal value. We use Gurobi \cite{gurobi} to solve the MICPs. The wall-clock solve time is capped at 600 s. Runs exceeding this limit are reported as timeouts.

Table~\ref{tab:stl-robustness-time-compact} reports the optimal values and solve times. MICP is guaranteed to find the global optimal solutions, but it can easily run into a timeout due to the combinatorial complexity. Instead, both NLP-based methods are more scalable than MICP. Among them, our method consistently attains better optimal values than \cite{gilpin2020smooth}. At $T=25$, we warm-start both NLP-based methods with the globally optimal MICP solutions. Interestingly, our method preserves these optimal values, whereas \cite{gilpin2020smooth} often degrades due to the approximation error. For $T>25$, the optimal values of our method can underperform MICP. This is because our method is always warm-started from the $T=25$ MICP solutions via interpolation, which are no longer optimal with a larger $T$. Regarding runtime, our method is very fast when the initial guesses are good. However, this speed advantage does not always persist. While customizing solvers may further improve the speed, it is beyond the scope of this paper.

\begin{table}[t]
\centering
\footnotesize
\setlength{\tabcolsep}{4.5pt}
\caption{Results of linear scenarios. (\text{-}) indicates a timeout.}
\label{tab:stl-robustness-time-compact}
\begin{tabular}{@{}l c c c c c c c@{}}
\toprule
\multirow{2}{*}{Scenario} & \multirow{2}{*}{$T$} &
\multicolumn{3}{c}{Optimal Value  $\downarrow$} & \multicolumn{3}{c}{Solve Time (s)  $\downarrow$} \\
\cmidrule(l){3-5}\cmidrule(l){6-8}
& & {MICP} & {Ours} & {\cite{gilpin2020smooth}} & {MICP} & {Ours} & {\cite{gilpin2020smooth}} \\
\midrule
\multirow{3}{*}{Two-Target}
& 25  & {\bfseries 3.94} & {\bfseries 3.94} & 4.00 & 0.39 & {\bfseries 0.01} & 0.66 \\
& 50  & {\bfseries 1.67} &  1.68 & 1.70 & 469.52 & {\bfseries 5.14} & 5.24 \\
& 75  & \text{-} & {\bfseries 0.96} & 0.98 & \text{-} & 11.56 & {\bfseries 10.35} \\
\midrule
\multirow{3}{*}{Many-Target}
& 25  & {\bfseries 6.94} & {\bfseries 6.94} & 7.08 & 2.39 & {\bfseries 0.01} & 1.21 \\
& 50  & \text{-} & {\bfseries 3.32} & 3.37 & \text{-} & 9.79 & {\bfseries 8.10} \\
& 75  & \text{-} & {\bfseries 2.18} & 2.22 & \text{-} & 35.55 & {\bfseries 21.54} \\
\midrule
\multirow{3}{*}{Narrow-Passage}
& 25  & {\bfseries 1.83} & {\bfseries 1.83} & 1.89 & 1.30 & {\bfseries 0.01} & 0.30 \\
& 50  & {\bfseries 0.88} & {\bfseries 0.88} & 0.92 & 82.22 & 8.70 & {\bfseries 2.70} \\
& 75  & \text{-} & {\bfseries 0.53} & 0.58 & \text{-} & 32.89 & {\bfseries 8.09} \\
\midrule
\multirow{3}{*}{Door-Puzzle}
& 25  & {\bfseries 27.69} & {\bfseries 27.69} & 30.10 & 3.68 & {\bfseries 0.03} & 0.97 \\
& 50  & \text{-} & {\bfseries 12.58} & 14.00 &  \text{-} & 568.83 & {\bfseries 9.62} \\
& 75  & \text{-} & \text{-} & {\bfseries 8.86} & \text{-} & \text{-} & {\bfseries 47.50} \\
\bottomrule
\end{tabular}
\end{table}

\subsection{Nonlinear Case Study}
We then validate the proposed method on a scenario with nonlinear dynamics and nonlinear predicates. The state and control input are defined as
$$
\xb_t=[p_{x_t},p_{y_t},\theta_t]^\top \in \mathbb{R}^3,\quad
\ub_t=[v_t,\omega_t]^\top \in \mathbb{R}^2,
$$
where $p_{x_t}$ and $p_{y_t}$ denote the robot’s horizontal and vertical positions, and $\theta_t$ is the heading. 
The system follows a kinematic unicycle model,
$$
\xb_{t+1}=\fb(\xb_t,\ub_t)=
\begin{bmatrix}
p_{x_t}+\Delta t\, v_t\cos\theta_t\\
p_{y_t}+\Delta t\, v_t\sin\theta_t\\
\theta_t+\Delta t\, \omega_t
\end{bmatrix},
$$
where $\Delta t=0.5$ is the sampling time. We set the time horizon $T=50$, and the STL specification is defined as:
\begin{equation*}
\begin{aligned}
\varphi =
&\always_{[0,2]}\eventually_{[5,7]}\mu_{1}
\land \eventually_{[15,17]}\always_{[2,5]}\mu_{2} \\
&\land \eventually_{[27,35]}\bigl(\mu_{3}\,\until_{[2,10]}\,\mu_{1}\bigr)
\land \always_{[0,50]}\neg\mu_{4}
\land \eventually_{[37,50]}\mu_{5},
\end{aligned}
\end{equation*}
and ${h^{\mu_i}=r_i^2-(p_{x_t}-c_{i,x})^2-(p_{y_t}-c_{i,y})^2}$, where $r_i$ and $[c_{i,x},c_{i,y}]^{\top}$ are the radius and center. We note that this specification cannot be solved by the convex-concave procedure (CCP)-based method in \cite{takayama2025stlccp} since it requires predicate functions to be convex when the parent node is of type $\max$. Instead, our formulation has no such restriction.

For general nonlinear settings, it is non-trivial to find effective initializations for the NLPs.  We therefore use randomized initial guesses drawn from a normal distribution. We set ${\alpha=10}$, $\bm Q=\operatorname{diag}(0,0,0)$, and $\bm R=\operatorname{diag}(0.1,1)$, and constrain the control inputs by $|v_t|\le 1$ and $|\omega_t|\le 1$. The initial state $\xb_0=[2, 3, -\pi/2]^{\top}$. We compare our method with \cite{gilpin2020smooth} over $100$ random seeds. We pick $k=25$ for the baseline method. The trajectory achieving the lowest objective value is shown in Fig.~\ref{fig:nonlinear_trajs}, and corresponding metrics are summarized in Table~\ref{tab:nonlinear-results}. 

From the results, we can observe that our method is able to attain a better optimal value, but exhibits a slightly higher infeasibility rate. This may be mitigated with improved initializations. We leave the development of systematic initialization strategies for future work.

\begin{figure}[t]
    \centering
    \includegraphics[width=1\linewidth]{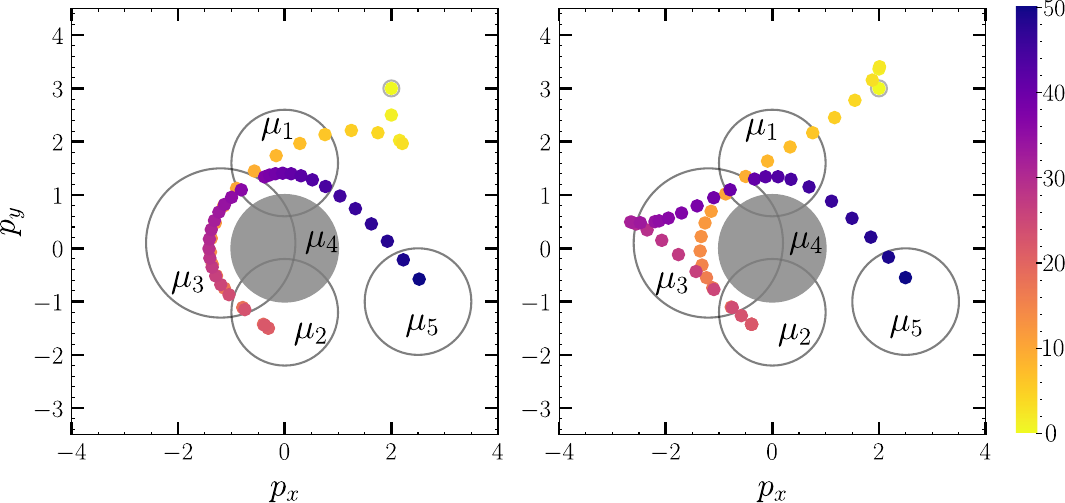}
    \caption{Comparison of trajectories generated by the baseline method \cite{gilpin2020smooth} (left) and our approach (right). The color bar indicates the time steps along each trajectory.}
    \label{fig:nonlinear_trajs}
\end{figure}

\begin{table}[th]
\centering
\footnotesize
\caption{Results of the nonlinear scenario. The robustness is computed using the discrete $\max$ and $\min$, and infeasible indicates the solver failed to find a solution given a random initial guess.}
\label{tab:nonlinear-results}
\begin{tabular}{@{}l c c c c@{}}
\toprule
Method & {Optimal Value} & {Robustness $\uparrow$} & {Solve Time (s)} & {Infeasible} \\
\midrule
Ours & {\bfseries -1.65} & {\bfseries 0.80} & {\bfseries 3.17} & 19/100 \\
\cite{gilpin2020smooth} & -0.86 & 0.79 & 11.03 & {\bfseries 14/100} \\
\bottomrule
\end{tabular}
\end{table}

\section{Conclusion}
\label{sec:conclusion}
In this paper, we take a further step in STL trajectory optimization by providing exact smooth reformulations of the $\max$ and $\min$ operators. Our approach accommodates general nonlinear dynamics and nonlinear predicates while being exact in quantifying the robustness. We validate the method on linear benchmarks and a nonlinear example. Future work will explore informed warm-start techniques for the NLPs.

\section*{Acknowledgements}
We thank Luyao Zhang, Chelsea Sidrane, and Luzia Knoedler for the helpful discussions.
    
\bibliographystyle{IEEEtran}
\bibliography{references}

\end{document}